\documentclass{article} % For LaTeX2e
\usepackage{nips15submit_e,times}
\usepackage{hyperref}
\usepackage{url}

\usepackage{amsmath}
\usepackage{amsthm}
\usepackage{wrapfig}
\usepackage{algorithm}
\usepackage{algorithmic}
\usepackage{graphicx}

\newtheorem{lemma}{Lemma}
\newtheorem{corollary}{Corollary}
\newtheorem{theorem}{Theorem}

\usepackage{amssymb}
\usepackage{bm}
\usepackage{xspace}
\usepackage[colorinlistoftodos, textwidth=40mm, shadow]{todonotes}
\usepackage[numbers]{natbib}

\newcommand{\loss}{\ell}
\newcommand{\hloss}{\wh{\ell}}
\newcommand{\tloss}{\wt{\ell}}
\newcommand{\rew}{r}
\newcommand{\hrew}{\wh{\rew}}
\newcommand{\trew}{\wt{\rew}}

\newcommand{\real}{\mathbb{R}}

\newcommand{\OO}{\mathcal{O}}
\newcommand{\tOO}{\wt{\OO}}
\newcommand{\II}[1]{\mathbb{I}_{\left\{#1\right\}}}
\newcommand{\PP}[1]{\mathbb{P}\left[#1\right]}
\newcommand{\EE}[1]{\mathbb{E}\left[#1\right]}

\newcommand{\PPcc}[2]{\mathbb{P}\left[\left.#1\right|#2\right]}

\newcommand{\EEcc}[2]{\mathbb{E}\left[\left.#1\right|#2\right]}

\newcommand{\ev}[1]{\left\{#1\right\}}
\newcommand{\pa}[1]{\left(#1\right)}
\newcommand{\bpa}[1]{\bigl(#1\bigr)}

\newcommand{\F}{\mathcal{F}}

\renewcommand{\th}{\ensuremath{^{\mathrm{th}}}}
\def\argmin{\mathop{\rm arg\, min}}

\newcommand{\bV}{\boldsymbol{V}}
\newcommand{\be}{\boldsymbol{e}}

\newcommand{\bloss}{\bm{\ell}}

\newcommand{\xp}{\xi}
\newcommand{\bxp}{\bm{\xp}}

\newcommand{\bp}{\boldsymbol{p}}

\newcommand{\wh}{\widehat}
\newcommand{\wt}{\widetilde}

\newcommand{\exph}{\textsc{Exp3}\xspace}
\newcommand{\exphix}{\textsc{Exp3-IX}\xspace}
\newcommand{\exphp}{\textsc{Exp3.P}\xspace}
\newcommand{\expn}{\textsc{Exp4}\xspace}
\newcommand{\expnix}{\textsc{Exp4-IX}\xspace}

\newcommand{\exphsix}{\textsc{Exp3-SIX}\xspace}

\newcommand{\ra}{\rightarrow}
\newcommand{\transpose}{^\mathsf{\scriptscriptstyle T}}

\usepackage{todonotes}
\definecolor{PalePurp}{rgb}{0.66,0.57,0.66}

\title{Explore no more: Improved high-probability regret bounds for non-stochastic bandits}

\author{
Gergely Neu\thanks{The author is currently with the Department of Information and Communication Technologies, Pompeu 
Fabra University, Barcelona, Spain.} \\
SequeL team \\
INRIA Lille -- Nord Europe\\
\texttt{gergely.neu@gmail.com}
}

\nipsfinalcopy % Uncomment for camera-ready version

\begin{document}

\maketitle

\begin{abstract}
This work addresses the problem of regret minimization in non-stochastic multi-armed bandit problems, focusing on 
performance guarantees that hold with high probability.
Such results are rather scarce in the literature since proving them requires a large deal of technical effort and 
significant modifications to the standard, more intuitive algorithms that come only with guarantees that hold on 
expectation. 
One of these modifications is forcing the learner to sample arms from the uniform distribution at least 
$\Omega(\sqrt{T})$ times over $T$ rounds, which can adversely affect performance if many of the arms are suboptimal. 
While it is widely conjectured that this property is essential for proving high-probability regret bounds, we show in 
this paper that it is possible to achieve such strong results without this undesirable exploration component. Our result 
relies on a simple and intuitive loss-estimation strategy called \emph{Implicit eXploration} (IX) that allows a 
remarkably clean analysis. To demonstrate the flexibility of our technique, we derive several improved high-probability 
bounds for various extensions of the standard multi-armed bandit framework.
Finally, we conduct a simple experiment that illustrates the robustness of our implicit exploration technique.
\end{abstract}

\section{Introduction}
Consider the problem of regret minimization in non-stochastic multi-armed bandits, as defined in the classic
paper of \citet*{auer2002bandit}. This sequential decision-making problem can be formalized as a repeated game between 
a \emph{learner} and an \emph{environment} (sometimes called the \emph{adversary}). In each round $t=1,2,\dots,T$, the 
two players interact as
follows: The learner picks an \emph{arm} (also called an \emph{action}) $I_t\in[K] = \ev{1,2,\dots,K}$ and the
environment selects a loss function $\loss_t: [K] \ra [0,1]$, where the loss associated with arm $i\in[K]$ is denoted as
$\loss_{t,i}$. Subsequently, the learner incurs and observes the loss $\loss_{t,I_t}$. Based solely on these 
observations, the goal of the learner is to choose its actions so as to accumulate as little loss as possible during 
the course of the game. As traditional in the online learning literature \cite{CBLu06:book}, we measure the performance 
of the learner in  terms of the \emph{regret} defined as
\[
 R_T = \sum_{t=1}^T \loss_{t,I_t} - \min_{i\in[K]}\sum_{t=1}^T \loss_{t,i}.
\]
We say that the environment is \emph{oblivious} if it selects the sequence of loss vectors irrespective of the past
actions taken by the learner, and \emph{adaptive} (or \emph{non-oblivious}) if it is allowed to choose $\loss_t$ as a 
function of the past actions $I_{t-1},\dots,I_1$. An equivalent formulation of the multi-armed bandit game uses the 
concept of \emph{rewards} (also called \emph{gains} or \emph{payoffs}) instead of losses: in this version, the adversary 
chooses the sequence of \emph{reward functions} $(r_t)$ with $r_{t,i}$ denoting the reward given to the learner for 
choosing action $i$ in round $t$. In this game, the learner aims at maximizing its total rewards. We will refer to the 
above two formulations as the \emph{loss game} and the \emph{reward game}, respectively.

Our goal in this paper is to construct algorithms for the learner that guarantee that the regret grows sublinearly. 
Since it is well known that no deterministic learning algorithm can achieve this goal \citep{CBLu06:book}, we are 
interested in \emph{randomized} algorithms. Accordingly, the regret $R_T$ then becomes a random variable that we need 
to bound in some probabilistic sense.
Most of the existing literature on non-stochastic bandits is concerned with bounding 
the \emph{pseudo-regret} (or \emph{weak regret}) defined as
\[
 \wh{R}_T = \max_{i\in[K]}\EE{\sum_{t=1}^T \loss_{t,I_t} - \sum_{t=1}^T \loss_{t,i}},
\]
where the expectation integrates over the randomness injected by the learner.
Proving bounds on the actual regret that hold with high probability is considered to be a significantly harder task
that 
can be achieved by serious changes made to the learning algorithms and much more complicated analyses. One particular 
common belief is that in order to guarantee high-confidence performance guarantees, the learner cannot avoid repeatedly 
sampling arms from a uniform distribution, typically $\Omega\bpa{\sqrt{KT}}$ times 
\cite{auer2002bandit,audibert10inf,BLLRS11,bubeck12survey}. It is easy to see that such 
\emph{explicit exploration} can impact the empirical performance of learning algorithms in a very negative way if there 
are many arms with high losses: even if the base learning algorithm quickly learns to focus on good arms, explicit 
exploration still forces the regret to grow at a steady rate. As a result, algorithms with high-probability performance 
guarantees tend to perform poorly even in very simple problems \citep{SCALS12,BLLRS11}.

% For the sake of concreteness, consider the \exph algorithm of \citet{auer2002bandit} that comes with a strong bound on 
% the expected regret and its high-confidence variant \exphp. While theory suggests that \exphp should be more reliable 
% than vanilla \exph, practice shows the exact opposite: \citet{SCALS12} conduct some experiments on a simple 2-armed 
% bandit problem where \exph beats \exphp by a spectacular margin. A similar conclusion can be 
% drawn from the experiments of \citet{BLLRS11}, who evaluate \expn and \expnp in an experiment involving real-world 
% data and show that the online performance of the high-confidence learning algorithm \expnp is inferior to 
% the basic version of \expn. However, the experiment of \citeauthor{BLLRS11} also shows a favorable 
% property of \expnp: its \emph{offline} performance on held-out data is significantly better than that of plain \expn. 
% Still, the question remains: is it possible to demonstrate the advantage of high-confidence algorithms over simpler ones 
% in a true online sense? 

In the current paper, we propose an algorithm that guarantees strong regret bounds that hold with high 
probability without the explicit exploration component. 
One component that we preserve from the classical recipe for such algorithms is the \emph{biased estimation of losses}, 
although our bias is of a much more delicate nature, and arguably more elegant than previous approaches. In particular, 
we adopt the \emph{implicit exploration} (IX) strategy first proposed by \citet*{KNVM14} for the problem of online 
learning with side-observations. As we show in the current paper, this simple loss-estimation strategy allows proving 
high-probability bounds for a range of non-stochastic bandit problems including bandits with expert advice, tracking 
the best arm and bandits with side-observations. Our proofs are arguably cleaner and less involved than previous ones, 
and very elementary in the sense that they do not rely on advanced results from probability theory like Freedman's 
inequality \cite{Fre75}.
The resulting bounds are tighter than all previously known bounds and 
hold simultaneously for all confidence levels, unlike most previously known bounds \citep{auer2002bandit,BLLRS11}. For 
the first time in the literature, we also provide high-probability bounds for anytime algorithms that do not require 
prior knowledge of the time horizon $T$. A minor conceptual improvement in our analysis is a direct treatment of the 
loss game, as opposed to previous analyses that focused on the reward game, making our treatment more coherent with 
other state-of-the-art results in the online learning literature\footnote{In fact, studying the loss game is 
colloquially known to allow better constant factors in the bounds in many settings (see, e.g., 
\citet{bubeck12survey}). Our result further reinforces these observations.}.

The rest of the paper is organized as follows. In Section~\ref{sec:est}, we review the known techniques for 
proving high-probability regret bounds for non-stochastic bandits and describe our implicit exploration strategy in 
precise terms. Section~\ref{sec:apps} states our main result concerning the concentration of the IX loss estimates and 
shows applications of this result to several problem settings. Finally, we conduct a set of simple experiments to
illustrate the benefits of implicit exploration over previous techniques in Section~\ref{sec:exp}.

\section{Explicit and implicit exploration}\label{sec:est}
Most principled learning algorithms for the non-stochastic bandit problem are constructed by using a standard online
learning algorithm such as the exponentially weighted forecaster (\cite{Vov90,LiWa94,FS97}) or follow the perturbed
leader (\cite{Han57,KV05}) as a black box, with the true (unobserved) losses replaced by some appropriate estimates.
One of the key challenges is constructing reliable estimates of the losses $\loss_{t,i}$ for all $i\in[K]$ based on the
single observation $\loss_{t,I_t}$. Following \citet{auer2002bandit}, this is traditionally achieved by using
importance-weighted loss/reward estimates of the form
\begin{equation}\label{eq:tradest}
 \hloss_{t,i} = \frac{\loss_{t,i}}{p_{t,i}} \II{I_t = i} 
 \qquad\mbox{or}\qquad
 \hrew_{t,i} = \frac{\rew_{t,i}}{p_{t,i}} \II{I_t = i}
\end{equation}
where $p_{t,i} = \PPcc{I_t = i}{\F_{t-1}}$ is the probability that the learner picks action $i$ in round $t$, 
conditioned on the observation history $\F_{t-1}$ of the learner up to the beginning of round $t$.
It is easy to show that these estimates are unbiased for all $i$ with $p_{t,i}>0$ in the sense that
$\mathbb{E}\hloss_{t,i} = \loss_{t,i}$ for all such $i$. 

For concreteness, consider the \exph algorithm of \citet{auer2002bandit} as described in
\citet[Section~3]{bubeck12survey}. In every round $t$, this algorithm uses the loss estimates defined in
Equation~\eqref{eq:tradest} to compute the \emph{weights} $w_{t,i} = \exp\bpa{-\eta\sum_{s=1}^{t-1} \hloss_{s-1,i}}$ for
all $i$ and some positive parameter $\eta$ that is often called the \emph{learning rate}. Having computed these weights,
\exph draws arm $I_t=i$ with probability proportional to $w_{t,i}$. Relying on the unbiasedness of
the estimates~\eqref{eq:tradest} and an optimized setting of $\eta$, one can prove that \exph enjoys a pseudo-regret
bound of $\sqrt{2TK\log K}$. However, the fluctuations of the loss estimates around the true losses are too large to
permit bounding the true regret with high probability.
To keep these fluctuations under control, \citet{auer2002bandit} propose to use the \emph{biased reward-estimates}
\begin{equation}\label{eq:trew1}
 \trew_{t,i} = \hrew_{t,i}  + \frac{\beta}{p_{t,i}}
\end{equation}
with an appropriately chosen $\beta>0$. 
Given these estimates, the \exphp algorithm of \citet{auer2002bandit} computes the weights
$
 w_{t,i} = \exp\bpa{\eta \sum_{s=1}^{t-1}\trew_{s,i}}
$
for all arms $i$ and then samples $I_t$ according to the distribution
\[
 p_{t,i} = (1-\gamma) \frac{w_{t,i}}{\sum_{j=1}^K w_{t,j}} + \frac{\gamma}{K},
\]
where $\gamma \in[0,1]$ is the exploration parameter. The argument for this \emph{explicit exploration} is that it 
helps to keep the range (and thus the variance) of the above reward estimates bounded, thus enabling the use of (more 
or less) standard 
concentration results\footnote{Explicit exploration is believed to be inevitable for proving bounds in the reward game 
for various other reasons, too---see \citet{bubeck12survey} for a discussion.}. 
In particular, the key element in the analysis of \exphp \cite{auer2002bandit,bubeck12survey,BLLRS11,BaDaHaKaRaTe08} is
showing that the inequality
\[
 \sum_{t=1}^T \pa{\rew_{t,i} - \trew_{t,i}} \le \frac{\log (K/\delta)}{\beta}
\]
holds simultaneously for all $i$ with probability at least $1-\delta$. 
In other words, this shows that the cumulative estimates $\sum_{t=1}^T\trew_{t,i}$ are upper confidence 
bounds for the true rewards $\sum_{t=1}^T\rew_{t,i}$.

In the current paper, we propose to use the loss estimates defined as
\begin{equation}\label{eq:ix}
 \tloss_{t,i} = \frac{\loss_{t,i}}{p_{t,i} + \gamma_t} \II{I_t = i},
\end{equation}
for all $i$ and an appropriately chosen $\gamma_t > 0$, and then use the resulting estimates in an 
exponential-weights algorithm scheme without any explicit exploration.
Loss estimates of this form were first used by \citet{KNVM14}---following them, we refer to this technique as
\emph{Implicit eXploration}, or, in short, IX.
In what follows, we argue that that IX as defined above achieves a similar variance-reducing effect as the one achieved 
by the combination of explicit exploration and the biased reward estimates of Equation~\eqref{eq:trew1}. In particular,
we show that the IX estimates~\eqref{eq:ix} constitute a lower confidence bound for the true losses which allows
proving high-probability bounds for a number of variants of the multi-armed bandit problem.
% To get an intuition, first observe that our estimates can be rewritten as
% \[
% \begin{split}
%  \tloss_{t,i} =& \frac{\loss_{t,i}\pa{\II{I_t = i} + \gamma \loss_{t,i}}}{p_{t,i} + \gamma \loss_{t,i}}  - 
%  \gamma \frac{\loss_{t,i}^2}{p_{t,i} + \gamma \loss_{t,i}}.
% \end{split}
% \]
% These two terms correspond to an \emph{unbiased} estimate of $\loss_{t,i}$ 
% and a bias term that bounds $\gamma$ times the variance of the first term. This allows a direct application of 
% Freedman's inequality to prove that
% \[
%  \sum_{t=1}^T \pa{\tloss_{t,i} - \loss_{t,i}} \le (e-2) \frac{\log\pa{K/\delta}}{\gamma},
% \]
% holds with probability at least $1-\delta$, similarly to the bound concerning the reward-estimates~\eqref{eq:trew1}. 
% In the next section, we make the above claim more formal and present an elementary proof not relying on Freedman's 
% inequality.
% Note however that the bias introduced by IX 
% can be much smaller than the bias terms used in~\eqref{eq:trew1} when the losses tend to be small.

\section{High-probability regret bounds via implicit exploration}\label{sec:apps}
In this section, we present a concentration result concerning the IX loss estimates of Equation~\eqref{eq:ix}, and 
apply this result to prove high-probability performance guarantees for a number of non-stochastic bandit problems. 
The following lemma states our concentration result in its most general form:
\begin{lemma}\label{lem:fixbound}
 Let $\pa{\gamma_t}$ be a \emph{fixed} non-increasing sequence with $\gamma_t\ge 0$ and let
 $\alpha_{t,i}$ be nonnegative $\F_{t-1}$-measurable random variables satisfying $\alpha_{t,i}\le 2\gamma_t$ for all 
$t$ and 
$i$. Then, with probability at least $1-\delta$,
 \[
  \sum_{t=1}^T \sum_{i=1}^K \alpha_{t,i} \pa{\tloss_{t,i} - \loss_{t,i}} \le \log\pa{1/\delta}.
 \]
\end{lemma}
A particularly important special case of the above lemma is the following:
\begin{corollary}\label{cor:allbound}
 Let $\gamma_t = \gamma \ge 0$ for all $t$. With probability at least $1-\delta$,
 \[
  \sum_{t=1}^T \pa{\tloss_{t,i} - \loss_{t,i}} \le \frac{\log\pa{K/\delta}}{2\gamma}.
 \]
 simultaneously holds for all $i\in[K]$.
\end{corollary}
This corollary follows from applying Lemma~\ref{lem:fixbound} to the functions $\alpha_{t,i} = 2\gamma\II{i = j}$ for 
all $j$ and 
applying the union bound. The full proof of Lemma~\ref{lem:fixbound} is presented in the Appendix. For didactic 
purposes, we now present a direct proof for Corollary~\ref{cor:allbound}, which is essentially a simpler version of 
Lemma~\ref{lem:fixbound}.
\begin{proof}[Proof of Corollary~\ref{cor:allbound}]
 For convenience, we will use the notation $\beta = 2\gamma$. First, observe that
 \[
  \begin{split}
   \tloss_{t,i} =& \frac{\loss_{t,i}}{p_{t,i} + \gamma} \II{I_t = i} 
   \le \frac{\loss_{t,i}}{p_{t,i} + \gamma \loss_{t,i}} \II{I_t = i} 
   = \frac{1}{2 \gamma} \cdot \frac{2 
\gamma \loss_{t,i} / p_{t,i}}{1 + \gamma \loss_{t,i} / p_{t,i}} \II{I_t = i}
   \le \frac{1}{\beta} \cdot \log\pa{1 + \beta \hloss_{t,i}},
  \end{split}
 \]
 where the first step follows from $\loss_{t,i}\in [0,1]$ and last one from the elementary inequality
$\frac{z}{1+z/2} \le \log(1+z)$ that holds for all $z\ge
0$. Using the above inequality, we get that
 \[
 \begin{split}
  \EEcc{\exp\pa{\beta \tloss_{t,i}}}{\F_{t-1}} \le& \EEcc{1+\beta
\hloss_{t,i}}{\F_{t-1}}
\le 1+\beta \loss_{t,i} \le \exp\pa{\beta \loss_{t,i}},
 \end{split}
 \]
 where the second and third steps are obtained by using $\EEcc{\hloss_{t,i}}{\F_{t-1}} \le \loss_{t,i}$ that holds 
by definition of $\hloss_{t,i}$, and the inequality $1+z \le e^z$ that holds for all $z\in\real$. As a
result, the process $Z_t = \exp\bpa{\beta \sum_{s=1}^t \bpa{\tloss_{s,i} - \loss_{s,i}}}$ is a supermartingale 
with respect to $\pa{\F_t}$: $ \EEcc{Z_t}{\F_{t-1}} \le Z_{t-1}$.
Observe that, since $Z_0 = 1$, this implies $\EE{Z_T} \le \EE{Z_{T-1}} \le \ldots \le 1$, and thus by Markov's
inequality,
\[
\begin{split}
 \PP{\sum_{t=1}^T\bpa{\tloss_{t,i}  - \loss_{t,i}}> \varepsilon} &\le \EE{\exp\pa{\beta 
\sum_{t=1}^T\bpa{\tloss_{t,i}  - \loss_{t,i}}}} \cdot \exp(-\beta \varepsilon) \le \exp(-\beta \varepsilon)
\end{split}
\]
holds for any $\varepsilon>0$. The statement of the lemma follows from solving $\exp(-\beta \varepsilon) = \delta/K$ 
for $\varepsilon$ and using the union bound over all arms $i$.
\end{proof}

In what follows, we put Lemma~\ref{lem:fixbound} to use and prove improved high-probability performance guarantees 
for several well-studied variants of the non-stochastic bandit problem, namely, the multi-armed bandit problem with 
expert advice, tracking the best arm for multi-armed bandits, and bandits with side-observations. The 
general form of Lemma~\ref{lem:fixbound} will allow us to prove high-probability bounds for anytime algorithms that can 
operate without prior knowledge of $T$. For clarity, we will only provide such bounds for the standard multi-armed 
bandit setting; extending the derivations to other settings is left as an easy exercise.
For all algorithms, we prove bounds that scale linearly with $\log(1/\delta)$ and hold
simultaneously for all levels $\delta$. 
Note that this dependence can be improved to $\sqrt{\log(1/\delta)}$ for a
fixed confidence level $\delta$, if the algorithm can use this $\delta$ to tune its parameters. 
This is the way that Table~\ref{tab:results} presents our new bounds side-by-side with the best previously known ones.

\begin{table}
\begin{center}
\begin{tabular}{|l|c|c|}
\hline
\bfseries{Setting} & \bfseries{Best known regret bound} & \bfseries{Our new regret  bound}\\
\hline
Multi-armed bandits & $5.15 \sqrt{TK\log (K/\delta)}$ & $2\sqrt{2TK\log (K/\delta)}$\\
Bandits with expert advice & $6 \sqrt{TK\log (N/\delta)}$ & $2\sqrt{2TK \log (N/\delta)}$\\
Tracking the best arm & $7\sqrt{KTS \log (KT/\delta S)}$ & $2\sqrt{2 KTS \log (KT/\delta S)}$\\
Bandits with side-observations & $\tOO\bpa{\sqrt{mT}}$ & $\tOO\bpa{\sqrt{\alpha T}}$\\
%Combinatorial semi-bandits & N/A & $2\sqrt{2mdT\log(ed/m)}$\\
\hline
\end{tabular}
\end{center}
\caption{Our results compared to the best previously known results in the four settings considered in
Sections~\ref{sec:exp3}--\ref{sec:exp3ix}. See the respective sections for references and notation.
}\label{tab:results}
\end{table}
\newpage
\subsection{Multi-armed bandits}\label{sec:exp3}
\begin{wrapfigure}{r}{0.5\textwidth}
\vspace{-.8cm}
\begin{minipage}{0.5\textwidth}
   \begin{algorithm}[H]
\caption{\exphix}
\label{alg:expix}
 \textbf{Parameters:} $\eta>0$, $\gamma>0$.\\
 \textbf{Initialization:} $w_{1,i} = 1$.\\
% \STATE \textbf{Initialization:} $\hLoss_{0,i} \gets 0$ for  $i\in [d]$
{\bfseries{for} $t = 1,2,\dots,T$, \bfseries{repeat}}
\begin{enumerate}
 \item $p_{t,i} = \frac{w_{t,i}}{\sum_{j=1}^K w_{t,j}}$.
 \item Draw $I_t  \sim \bp_t= (p_{t,1},\dots,p_{t,K})$.
\item  Observe loss $\loss_{t,I_t}$.
\item  $\tloss_{t,i} \gets \frac{\loss_{t,i}}{p_{t,i}+\gamma}\II{I_t = i}$ for all $i\in[K]$.
 \item   $w_{t+1,i} \gets w_{t,i} e^{-\eta\tloss_{t,i}}$ for all $i\in [K]$.
\end{enumerate}
\end{algorithm}
    \end{minipage}
    \vspace{-.8cm}
\end{wrapfigure}
In this section, we propose a variant of the \exph algorithm of \citet{auer2002bandit} that uses the IX 
loss estimates \eqref{eq:ix}: \exphix. The algorithm in its most general form uses two nonincreasing sequences of 
nonnegative parameters: $(\eta_t)$ and 
$(\gamma_t)$. In every round, \exphix chooses action $I_t = i$ with probability proportional to
\begin{equation}\label{eq:exph}
 p_{t,i} \propto w_{t,i} = \exp\pa{-\eta_t \sum_{s=1}^{t-1} \tloss_{s,i}},
\end{equation}
without mixing any explicit exploration term into the distribution. A fixed-parameter version of \exphix is presented as 
Algorithm~\ref{alg:expix}.

Our theorem below states a high-probability bound on the regret of \exphix. Notably, our bound exhibits the best known 
constant factor of $2\sqrt{2}$ in the leading term, improving on the 
factor of $5.15$ due to \citet{bubeck12survey}.  The best known leading constant for the pseudo-regret bound of \exph is
$\sqrt{2}$, also proved in \citet{bubeck12survey}.
\begin{theorem}\label{thm:main}
 Fix an arbitrary $\delta>0$. 
 With $\eta_t = 2\gamma_t = \sqrt{\frac{2\log K}{KT}}$ for all $t$, \exphix 
guarantees
 \[
  R_T \le 2\sqrt{2KT\log K} + \pa{\sqrt{\frac{2KT}{\log K}} + 1} \log\pa{2/\delta}
 \]
 with probability at least $1-\delta$.
 Furthermore, setting $\eta_t = 2\gamma_t = \sqrt{\frac{\log K}{Kt}}$ for all $t$, the bound becomes
 \[
  R_T \le 4\sqrt{KT\log K} + \pa{2\sqrt{\frac{KT}{\log K}} +1 } \log\pa{2/\delta}.
 \]
\end{theorem}
\begin{proof}
 Let us fix an arbitrary $\delta'\in(0,1)$.
 Following the  standard analysis of \exph in the loss game and nonincreasing learning rates \citep{bubeck12survey}, we 
can obtain the 
bound
 \[
  \sum_{t=1}^T \pa{\sum_{i=1}^K p_{t,i} \tloss_{t,i} - \tloss_{t,j}} \le \frac{\log K}{\eta_T} + 
  \sum_{t=1}^T  \frac{\eta_t}{2} \sum_{i=1}^K p_{t,i} \pa{\tloss_{t,i}}^2
 \]
 for any $j$.
 Now observe that
 \begin{equation}\label{eq:plossbound}
  \begin{split}
   \sum_{i=1}^K p_{t,i} \tloss_{t,i} 
   %&= 
   %\sum_{i=1}^K p_{t,i} \frac{\loss_{t,i}}{p_{t,i} + \gamma \loss_{t,i}} \II{I_t = i}
   %\\
   &= \sum_{i=1}^K \II{I_t = i} \frac{\loss_{t,i} \pa{p_{t,i} + \gamma_t }}{p_{t,i} + \gamma_t } - 
   \gamma_t \sum_{i=1}^K \II{I_t = i} \frac{\loss_{t,i}}{p_{t,i} + \gamma_t \loss_{t,i}}
   = 
   \loss_{t,I_t} - \gamma_t \sum_{i=1}^K \tloss_{t,i}.
  \end{split}
 \end{equation}
 Similarly, $\sum_{i=1}^K p_{t,i} \tloss_{t,i}^2 \le \sum_{i=1}^K \tloss_{t,i}$ holds by the boundedness of the losses.
 Thus, we get that
 \[
 \begin{split}
  \sum_{t=1}^T \pa{\loss_{t,I_t} - \loss_{t,j}} \le& \sum_{t=1}^T \pa{\loss_{t,j} - \tloss_{t,j}} + \frac{\log 
K}{\eta_T} + 
\sum_{t=1}^T \pa{\frac{\eta_t}{2} + \gamma_t} \sum_{i=1}^K \tloss_{t,i}
\\
\le& \frac{\log\pa{K/\delta'}}{2\gamma} +\frac{\log K}{\eta} + \sum_{t=1}^T \pa{\frac{\eta_t}{2}
+ \gamma_t} \sum_{i=1}^K \loss_{t,i} +  \log\pa{1/\delta'}
 \end{split}
 \]
 holds with probability at least $1-2 \delta'$, where the last line follows from an application of
Lemma~\ref{lem:fixbound} with $\alpha_{t,i} = \eta_t/2 + \gamma_t$ for all $t,i$ and taking the union bound.
By taking $j = \argmin_i L_{T,i}$ and $\delta' = \delta/2$, and using the boundedness of the losses, we obtain
 \[
  R_T \le \frac{\log\pa{2K/\delta}}{2\gamma_T}+ \frac{\log K}{\eta_T}  + K\sum_{t=1}^T \pa{\frac{\eta_t}{2}
+ \gamma_t} + \log\pa{2/\delta}.
 \]
 The statements of the theorem then follow immediately, noting that $\sum_{t=1}^T 1/\sqrt{t}\le 2\sqrt{T}$.
\end{proof}

\subsection{Bandits with expert advice}\label{sec:exp4}
We now turn to the setting of multi-armed bandits with expert advice, as defined in \citet{auer2002bandit}, and later 
revisited by \citet{McMaStre09} and \citet{BLLRS11}. In this setting, we assume that in every round $t=1,2,\dots,T$, 
the learner observes a set of $N$ probability distributions $\bxp_{t}(1),\bxp_{t}(2),\dots,\bxp_{t}(N)\in[0,1]^K$ over 
the $K$ arms, such that $\sum_{i=1}^K \xp_{t,i}(n) = 1$ for all $n\in[N]$. We assume that the sequences 
$\pa{\bxp_t(n)}$ are measurable with respect to $\pa{\F_{t}}$. The $n$\th of these vectors represent the 
probabilistic advice of the corresponding $n$\th~\emph{expert}. The goal of the learner in this setting is to pick a 
sequence of arms so as to minimize the regret against the best expert:
\[
 R_T^{\xp} = \sum_{t=1}^T \loss_{t,I_t} - \min_{n\in[N]} \sum_{t=1}^T \sum_{i=1}^K \xp_{t,i}(n)\loss_{t,i} \ra \min.
\]
To tackle this problem, we propose a modification of the \expn algorithm of 
\citet{auer2002bandit} that uses the IX loss estimates~\eqref{eq:ix}, and also 
drops the explicit exploration component of the original algorithm. Specifically, \expnix uses the loss estimates 
defined in Equation~\eqref{eq:ix} to compute the weights
\[
 w_{t,n} = \exp\pa{-\eta \sum_{s=1}^{t-1} \sum_{i=1}^K \xp_{s,i}(n)\tloss_{s,i}}
\]
for every expert $n\in[N]$, and then draw arm $i$ with probability $p_{t,i}\propto{\sum_{n=1}^N
w_{t,n} \xp_{t,i}(n)}$.
% A serious limitation of previous variants of \expn is that they have to ensure that $p_{t,i} = \Omega\bpa{\sqrt{\log K/ 
% (KT)}}$ holds for all $t$ and $i$, which is usually enforced by explicit exploration of actions according to the 
% uniform distribution. While \citet{McMaStre09} and \citet{BLLRS11} propose some alternative exploration schemes to 
% ensure the above condition, these solutions still rely in explicitly increasing the probability of drawing every single 
% action, which leads to serious performance degradation if there are many poorly performing arms. Our implicit 
% exploration technique again avoids the need to draw suboptimal arms, and allows proving a tighter high-probability 
% regret bound by a simple technique. 
We now state the performance guarantee of \expnix. Our bound improves the best known 
leading constant of $6$ due to \citet{BLLRS11} to $2\sqrt{2}$ and is a factor of $2$ worse
than the best known constant in the pseudo-regret bound for \expn \cite{bubeck12survey}. The proof of the theorem is
presented in the Appendix.
\begin{theorem}\label{thm:experts}
Fix an arbitrary $\delta>0$ and set $\eta = 2\gamma = \sqrt{\frac{2\log N}{KT}}$ for all $t$. Then, with 
probability at least $1-\delta$, the regret of \expnix satisfies
 \[
  R_T^\xi \le 2\sqrt{2KT\log N} + \pa{\sqrt{\frac{2KT}{\log N}} +1 }\log\pa{2/\delta}.
 \]
\end{theorem}

\subsection{Tracking the best sequence of arms}\label{sec:exp3s}
In this section, we consider the problem of competing with sequences of actions. Similarly to 
\citet{HW98}, we consider the class of sequences that switch at most $S$ times between actions. 
We measure the performance of the learner in this setting in terms of the regret against the best sequence from this 
class $C(S)\subseteq[K]^T$, defined as
\[
 R_T^S= \sum_{t=1}^T \loss_{t,I_t} - \min_{\pa{J_t}\in C(S)} \sum_{t=1}^T \loss_{t,J_t}.
\]
% The strongest bounds on the \emph{tracking regret} defined above were proven by \citet{AB09}, who show that
% \exph with reward estimates of the form~\eqref{eq:trew1} achieves a regret of order $\sqrt{KST\log\pa{KT/S}}$ with high 
% probability. Besides this result, the \exphs algorithm of \citet{auer2002bandit} (an adaptation of the Fixed Share 
% algorithm of \citet{HW98}) is also known to provide similar guarantees on the pseudo-regret. Both of these algorithms 
% rely on enforcing $p_{t,i} = \wt{\Omega}\bpa{\sqrt{S/(KT)}}$ by an explicit exploration technique.
% 
Similarly to \citet{auer2002bandit}, we now propose to adapt the Fixed Share algorithm of \citet{HW98} to our setting. 
Our algorithm, called \exphsix, updates a set of weights $w_{t,\cdot}$ over the arms in a recursive fashion. In the 
first round, \exphsix sets $w_{1,i} = 1/K$ for all $i$. In the following rounds, the weights are updated for every arm 
$i$ as
\[
 w_{t+1,i} = (1-\alpha) w_{t,i} \cdot e^{-\eta \tloss_{t,i}} + \frac{\alpha}{K} \sum_{j=1}^K w_{t,j} \cdot e^{-\eta 
\tloss_{t,j}}.
\]
In round $t$, the algorithm draws arm $I_t = i$ with probability $p_{t,i} \propto w_{t,i}$.
% Note that \exphsix still 
% samples arms according to the uniform distribution with probability $\alpha$ in each round, which seems to go against 
% our efforts of abolishing explicit exploration. However, as our analysis below shows, setting $\alpha$ as small as 
% $S/T$ is sufficient for proving strong bounds. Such a rate is obviously not detrimental to empirical performance, since 
% uniform exploration is only invoked $O_P(S)$ times. 
Below, we give the performance guarantees of \exphsix. Note that 
our leading factor of $2\sqrt{2}$ again improves over the best previously known leading factor of $7$, shown by 
\citet{AB09}. The proof of the theorem is given in the Appendix.
\begin{theorem}\label{thm:tracking}
 Fix an arbitrary $\delta>0$ and set $\eta = 2\gamma = \sqrt{\frac{2\bar{S}\log K }{KT}}$ and 
$\alpha = \frac{S}{T-1}$, where $\bar{S} = S + 1$. Then, with probability at least $1-\delta$, the regret of \exphsix 
satisfies
 \[
  R_T^S \le 2\sqrt{2KT\bar{S}\log \pa{\frac{eKT}{S}}} + \pa{\sqrt{\frac{2KT}{\bar{S}\log K}} + 1} \log\pa{2/\delta}.
 \]
\end{theorem}

\subsection{Bandits with side-observations}\label{sec:exp3ix}
Let us now turn to the problem of online learning in bandit problems in the presence of side observations, as defined
by \citet{MS11} and later elaborated by \citet{ACGM13}. In this setting, the learner and the environment
interact exactly as in the multi-armed bandit problem, the main difference being that in every round, the learner
observes the losses of some arms other than its actually chosen arm $I_t$. The structure of the side observations is 
described by the directed graph $G$: nodes of $G$ correspond to individual arms, and the presence of
arc $i\ra j$ implies that the learner will observe $\loss_{t,j}$ upon selecting $I_t = i$. 

Implicit exploration and \exphix was first proposed by \citet{KNVM14} for this precise setting.
To describe this variant, let us introduce the notations $O_{t,i} = 
\II{I_t=i} + \II{(I_t\ra i)\in G}$ and $o_{t,i} = \EEcc{O_{t,i}}{\F_{t-1}}$.
Then, the IX loss estimates in this setting are defined for all $t,i$ as $ \tloss_{t,i}
= \frac{O_{t,i}\loss_{t,i}}{o_{t,i} + \gamma_t}$.
With these estimates at hand, \exphix draws arm $I_t$ from the exponentially weighted distribution defined in
Equation~\eqref{eq:exph}. The following theorem provides the regret bound concerning this algorithm.
\begin{theorem}\label{thm:sideobs}
 Fix an arbitrary $\delta>0$. Assume that $T\ge K^2/(8\alpha)$ and set $\eta = 2\gamma = \sqrt{\frac{\log 
K}{2\alpha T \log(KT)}}$, where 
$\alpha$ is the \emph{independence number} of $G$. With  probability at least $1-\delta$, \exphix guarantees
 \[
 \begin{split}
  R_T \le \pa{4 \!+\! 2\sqrt{\log\pa{4/\delta}}}\!\cdot\!\sqrt{2\alpha T \!\pa{\log^2\!\! K \!+\! \log KT}} \!+\! 
2\sqrt{\frac{\alpha T\log(KT)}{\log K}}\log\pa{4/\delta} \!+\! \sqrt{\frac{T\log(4/\delta)}{2}}.
 \end{split}
 \]
\
\end{theorem}
The proof of the theorem is given in the Appendix. While the proof of this statement is significantly more involved
than the other proofs presented in this paper, it provides a fundamentally new result. In particular, our bound is in
terms of the \emph{independence number} $\alpha$ and thus matches the minimax regret bound proved by \citet{ACGM13} for
this setting up to logarithmic factors. In contrast, the only high-probability regret bound for this setting due 
to \citet{ACGMMS14} scales with the size $m$ of the maximal acyclic subgraph of $G$, which can be 
much larger than $\alpha$ in general (i.e., $m$ may be $o(\alpha)$ for some graphs \cite{ACGM13}).

\vspace{-.2cm}
\section{Empirical evaluation}\label{sec:exp}
\vspace{-.2cm}
We conduct a simple experiment to demonstrate the robustness of \exphix as compared 
to \exph and its superior performance as compared to \exphp. Our setting is a 10-arm bandit problem where all losses are 
independent draws of Bernoulli random variables. The mean losses of arms 1 through 8 are $1/2$ and the mean loss of arm 
9 is $1/2 - \Delta$ for all rounds $t=1,2,\dots,T$. The mean losses of arm 10 are changing over time: for rounds 
$t\le T/2$, the mean is $1/2 + \Delta$, and $1/2 - 4\Delta$ afterwards. This choice ensures that up to at least 
round $T/2$, arm 9 is clearly better than other arms. In the second half of the game, arm 10 starts to outperform arm 9
 and eventually becomes the leader.

\begin{figure}
\begin{center}
\includegraphics[width=.49\columnwidth]{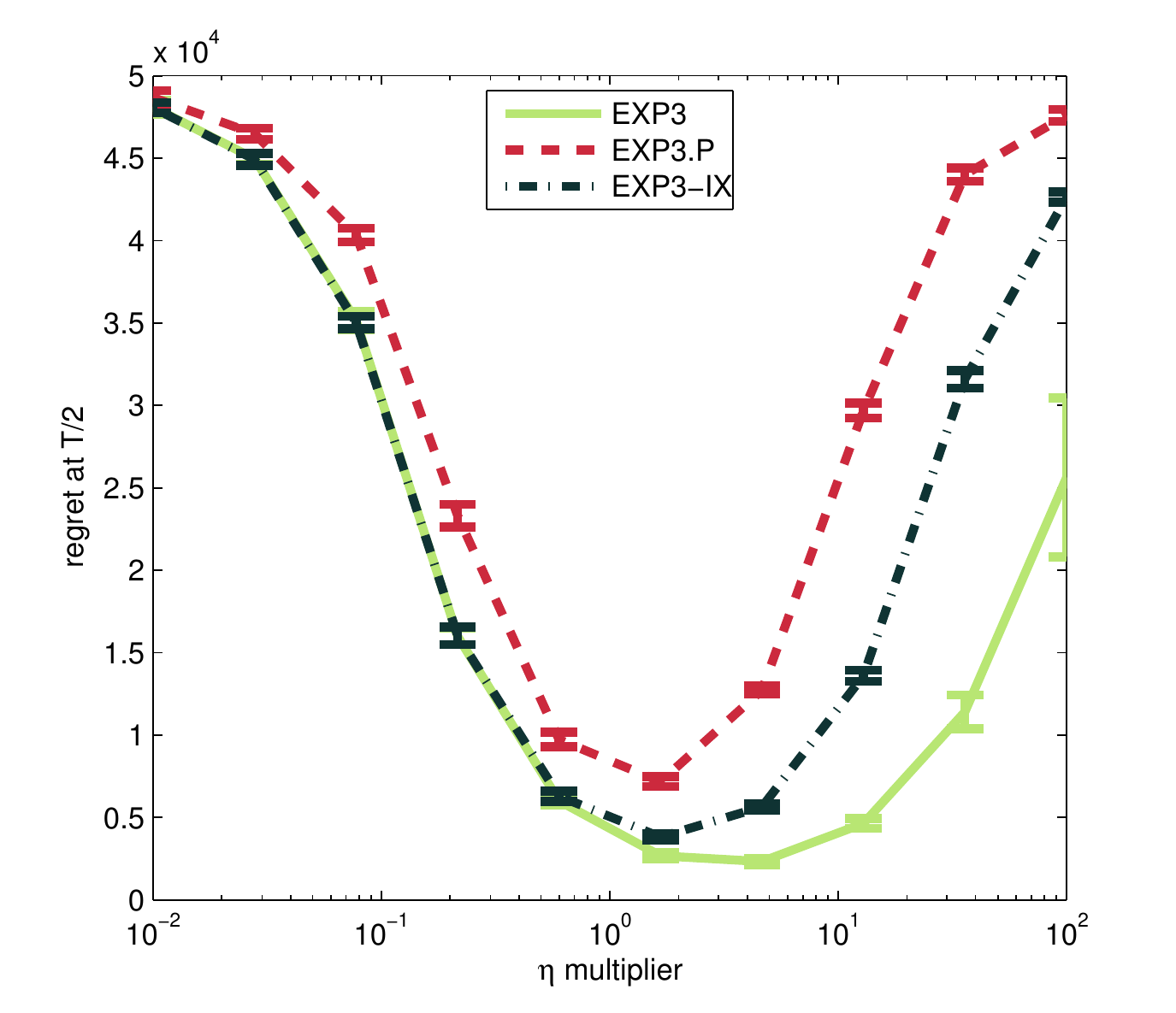}
\includegraphics[width=.49\columnwidth]{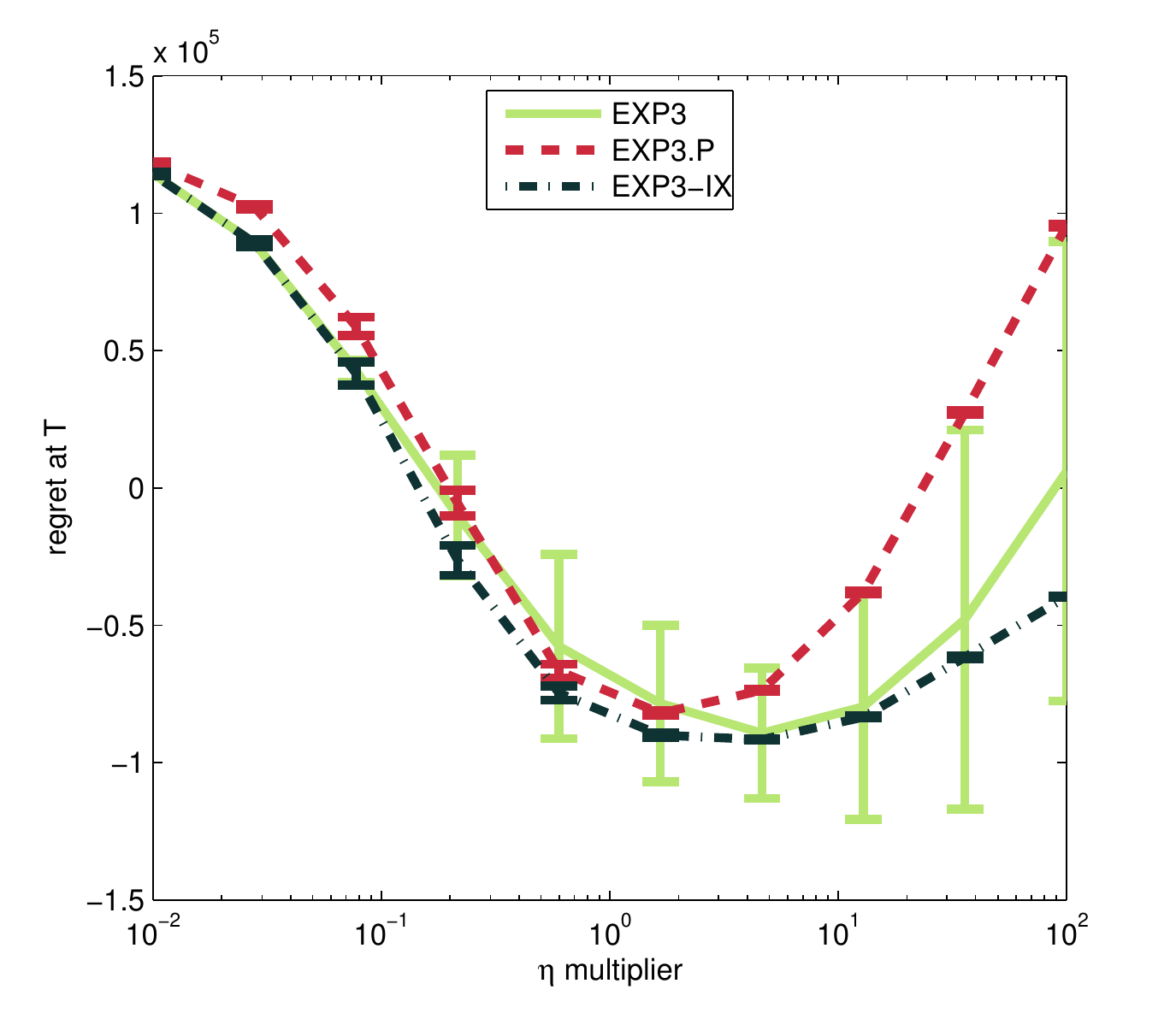}
\end{center}
\caption{Regret of \exph, \exphp, and \exphix, respectively in the problem described in Section~\ref{sec:exp}.}
\label{fig:exp}
\end{figure}

We have evaluated the performance of \exph, \exphp and \exphix in the above setting with $T = 10^6$ and 
$\Delta = 0.1$. For fairness of comparison, we evaluate all three algorithms for a wide range of parameters. 
In particular, for all three algorithms, we set a base learning rate $\eta$ according to the best known theoretical 
results \citep[Theorems~3.1 and~3.3]{bubeck12survey} and varied the multiplier of the respective base parameters 
between $0.01$ and $100$. Other parameters are set as $\gamma=\eta/2$ for \exphix and $\beta = \gamma/K = \eta$ for 
\exphp. We studied the regret up to two interesting 
rounds in the game: up to $T/2$, where the losses are i.i.d., and up to $T$ where the algorithms have to notice the 
shift in the loss distributions. Figure~\ref{fig:exp} shows the empirical means and standard deviations over 50 
runs of the regrets of the three algorithms as a function of the multipliers. 
The results clearly show that \exphix largely improves on the empirical performance of \exphp and is also much more 
robust in the non-stochastic regime than vanilla \exph. 
% We have run the experiment with 
% several settings of $\Delta$ and $T$ and experienced essentially the same type of behavior from all three algorithms.
% 
% In the first (i.i.d.) regime, \exph outperforms both algorithms for a wide range of parameter settings both in terms of 
% mean and variance. Reinforcing the conclusions of \citet{SCALS12}, \exphp is seen to perform much worse than \exph in 
% this regime. The performance of \exphix is essentially identical to that of \exph for small learning rates, but 
% eventually gets left behind for larger parameter values. 
% In the second regime, \exph becomes much more unreliable: while its mean regret is comparable to that of \exphix, it 
% misses the shift in the loss distributions much more often, resulting in a high variance.
% \exphp gathers more loss than its two competitors, but does so with little variation, as predicted by 
% theory. Notably, all three algorithms showed their best respective performances around the parameter 
% values suggested by theory in both regimes.

\vspace{-.3cm}
\section{Discussion}
\vspace{-.3cm}
In this paper, we have shown that, contrary to popular belief, explicit exploration is not necessary to achieve 
high-probability regret bounds for non-stochastic bandit problems. Interestingly, however, we have observed in several 
of our experiments that our IX-based algorithms still draw every arm roughly $\sqrt{T}$ times, even though this is not 
explicitly enforced by the algorithm. This suggests a need for a more complete study of 
the role of exploration, to find out whether pulling every single arm $\Omega(\sqrt{T})$ times is necessary for 
achieving near-optimal guarantees. 

One can argue that tuning the IX parameter that we introduce may actually be just as difficult in practice as tuning 
the parameters of \exphp. However, every aspect of our analysis suggests that $\gamma_t = \eta_t/2$ is the most natural 
choice for these parameters, and thus this is the choice that we recommend.
One limitation of our current analysis is that it only permits deterministic learning-rate and IX 
parameters (see the 
conditions of Lemma~\ref{lem:fixbound}). That is, proving adaptive regret bounds in the vein of 
\citep{HK11,RS13,Neu15} that hold with high probability is still an open challenge.

Another interesting question for future study is whether the implicit exploration approach can help in advancing the 
state of the art in the more general setting of linear bandits. All known algorithms for this 
setting rely on explicit exploration techniques, and the strength of the obtained results depend crucially on 
the choice of the exploration distribution (see \citep{BCK12,HKM14} for recent advances). Interestingly, IX has a 
natural extension to the linear bandit problem. To see this, consider the vector $\bV_t = \be_{I_t}$ and the matrix 
$P_t=\EE{\bV_t\bV_t\transpose}$. Then, the IX loss estimates can be written as $\wt{\bloss}_t =(P_t+\gamma 
I)^{-1}\bV_t\bV_t\transpose\bloss_t$. Whether or not this estimate is the right choice for linear bandits 
remains to be seen.

Finally, we note that our estimates~\eqref{eq:ix} are certainly not the only ones that allow avoiding explicit 
exploration. In fact, the careful reader might deduce from the proof of Lemma~\ref{lem:fixbound} that the same 
concentration bound can be shown to hold for the alternative loss estimates ${\loss_{t,i}\II{I_t = i}}/\pa{p_{t,i} + 
\gamma 
\loss_{t,i}} $ and $\log\bpa{1 + 2\gamma \loss_{t,i}\II{I_t=i}/p_{t,i}}/(2\gamma)$. Actually, a variant of the 
latter estimate was used previously for proving high-probability regret bounds in the reward game by 
\citet{audibert10inf}---however, their proof still relied on explicit exploration. It is not hard to verify that all 
the results we presented in this paper (except Theorem~\ref{thm:sideobs}) can be shown to hold for the above two 
estimates, too. 

\paragraph{Acknowledgments}
This work was supported by INRIA, the French Ministry of Higher Education and Research, and by FUI project Herm\`es. 
The author wishes to thank Haipeng Luo for catching a bug in an earlier version of the paper, and the 
anonymous reviewers for their helpful suggestions.

\newpage
{\footnotesize
% \bibliographystyle{abbrvnat}
% \bibliography{ngbib,allbib,predbook,shortconfs}

}

%\end{document}

\newpage
\appendix
\section{The proof of Lemma~\ref{lem:fixbound}}\label{sec:proof}
 Fix any $t$. For convenience, we will use the notation $\beta_t = 2\gamma_t$. First, observe that for any
$i$,
 \[
  \begin{split}
   \tloss_{t,i} =& \frac{\loss_{t,i}}{p_{t,i} + \gamma_t} \II{I_t = i} 
   \le \frac{\loss_{t,i}}{p_{t,i} + \gamma_t \loss_{t,i}} \II{I_t = i} 
   = \frac{1}{2 \gamma_t} \cdot \frac{2 
\gamma_t \loss_{t,i} / p_{t,i}}{1 + \gamma_t \loss_{t,i} / p_{t,i}} \II{I_t = i}
   \le \frac{1}{\beta_t} \cdot \log\pa{1 + \beta_t \hloss_{t,i}},
  \end{split}
 \]
 where the first step follows from $\loss_{t,i}\in [0,1]$ and last one from the elementary inequality
$\frac{z}{1+z/2} \le \log(1+z)$ that holds for all $z\ge
0$. 

Define the notations $\wt{\lambda}_t = \sum_{i=1}^K \alpha_{t,i} \tloss_{t,i}$ and 
$\lambda_t = \sum_{i=1}^K \alpha_{t,i} \loss_{t,i}$. Using the above inequality, we get that
 \begin{equation}\label{eq:incrbound}
 \begin{split}
  \EEcc{\exp\bpa{\wt{\lambda}_{t}}}{\F_{t-1}} \le& \EEcc{\exp\pa{\sum_{i=1}^K \frac{\alpha_{t,i}}{\beta_t} \cdot 
\log\pa{1+\beta_t
\hloss_{t,i}}}}{\F_{t-1}} 
\\
\le& \EEcc{\prod_{i=1}^K \pa{1 + \alpha_{t,i} \hloss_{t,i}}}{\F_{t-1}}
= \EEcc{{1 + \sum_{i=1}^K \alpha_{t,i} \hloss_{t,i}}}{\F_{t-1}}
\\
\le&
{1 + \sum_{i=1}^K \alpha_{t,i} \loss_{t,i}} \le \exp\pa{\sum_{i=1}^K \alpha_{t,i} \loss_{t,i}} = \exp\pa{
\lambda_{t}},
 \end{split}
 \end{equation}
 where the second line follows from noting that $\alpha_{t,i}\le \beta_t$, using the inequality $x\log(1+y) \le 
\log(1+xy)$ that holds for all $y > -1$ and $x\in[0,1]$ and the identity $\prod_{i=1}^K \pa{1 + \alpha_{t,i} 
\hloss_{t,i}} = 1 + 
\sum_{i=1}^K \alpha_{t,i} \hloss_{t,i}$ that follows from the fact that $\hloss_{t,i}\cdot\hloss_{t,j} = 0$ holds 
whenever $i\neq j$. The last line is obtained by using $\EEcc{\hloss_{t,i}}{\F_{t-1}} \le \loss_{t,i}$ that holds by 
definition of $\hloss_{t,i}$, and the inequality $1+z \le e^z$ that holds for all $z\in\real$. 

As a result, the process $Z_t = \exp\bpa{\sum_{s=1}^t \bpa{\wt{\lambda}_s - \lambda_{s}}}$ is a supermartingale 
with respect to $\pa{\F_t}$: $ \EEcc{Z_t}{\F_{t-1}} \le Z_{t-1}$.
Observe that, since $Z_0 = 1$, this implies $\EE{Z_T} \le \EE{Z_{T-1}} \le \ldots \le 1$, and thus by Markov's
inequality,
\[
\begin{split}
 \PP{\sum_{t=1}^T\bpa{\wt{\lambda}_{t}  - \lambda_{t}}> \varepsilon} &\le \EE{\exp\pa{
\sum_{t=1}^T\bpa{\wt{\lambda}_{t}  - \lambda_{t}}}} \cdot \exp(-\varepsilon) \le \exp(-\varepsilon)
\end{split}
\]
holds for any $\varepsilon>0$. The statement of the lemma follows from solving $\exp(-\varepsilon) = \delta$ for 
$\varepsilon$.
\qed 

\section{Further proofs}
\subsection{The proof of Theorem~\ref{thm:experts}}
 Fix an arbitrary $\delta'$. For ease of notation, let us define $\pi_t(n) = w_{t,n}/\bpa{\sum_{m=1}^N w_{t,m}}$.
 By standard arguments (along the lines of \cite{auer2002bandit,bubeck12survey}), we can obtain
 \[
  \sum_{t=1}^T \sum_{i=1}^K \bpa{p_{t,i} - \xp_{t,i}(m)} \tloss_{t,i} \le \frac{\log K}{\eta} + 
  \frac{\eta}{2} \sum_{t=1}^T  \sum_{n=1}^N \pi_{t}(n) \pa{\sum_{i=1}^K \xp_{t,i}(n)\tloss_{t,i}}^2
 \]
 for any fixed $m\in[N]$.
 The last term on the right-hand side can be bounded as
 \[
  \begin{split}
   &\sum_{n=1}^N \pi_{t}(n) \pa{\sum_{i=1}^K \xp_{t,i}(n)\tloss_{t,i}}^2 \le
   \sum_{n=1}^N \pi_{t}(n) \sum_{i=1}^K \xp_{t,i}(n)\pa{\tloss_{t,i}}^2
   =\sum_{i=1}^K p_{t,i} \pa{\tloss_{t,i}}^2 \le \sum_{i=1}^K \tloss_{t,i},
  \end{split}
 \]
 where the first step uses Jensen's inequality and the last uses $p_{t,i}\tloss_{t,i}\le 1$.
 Now, we can apply Lemma~\ref{lem:fixbound} and the union bound to show that
 \[
  \sum_{t=1}^T \sum_{i=1}^K \xp_{t,i}(m) \pa{\tloss_{t,i} - \loss_{t,i}} \le \frac{\log\pa{N/\delta'}}{2\gamma}
 \]
holds simultaneously for all experts with probability at least $1-\delta'$, and in particular for the best expert, too.
Putting this observation together 
with the above bound and Equation~\eqref{eq:plossbound}, we get that
 \[
 \begin{split}
  R_T^{\xp} \le& \frac{\log N}{\eta} + \frac{\log\pa{N/\delta'}}{2\gamma} + 
\pa{\frac{\eta}{2}
+ \gamma} \sum_{t=1}^T \sum_{i=1}^K \tloss_{t,i}
\\
\le& \frac{\log K}{\eta} + \frac{\log\pa{N/\delta'}}{2\gamma} + \pa{\frac{\eta}{2}
+ \gamma} \sum_{i=1}^K L_{T,i} + \pa{\frac{\eta}{2} + \gamma} \frac{\log\pa{1/\delta'}}{2\gamma}
 \end{split}
 \]
 holds with probability at least $1-2\delta'$, 
 where the last line follows from Lemma~\ref{lem:fixbound} and the union bound. The proof is concluded by taking
$\delta' = \delta/2$ and plugging in the choices of $\gamma$ and $\eta$.
\qed

\subsection{The proof of Theorem~\ref{thm:tracking}}
 The proof of the theorem builds on the techniques of \citet{CBGLS12} and \citet{auer2002bandit}. Let us fix an
arbitrary $\delta'\in(0,1)$ and denote the 
best sequence from $C(S)$ by $J_{1:T}^*$. Then, a straightforward modification of 
Theorem~2 of \cite{CBGLS12} yields the bound\footnote{Proving this bound requires replacing Hoeffding's inequality in 
their Lemma~1 by the inequality $e^{-z}\le 1 - z + {z^2/2}$ that holds for all $z\ge0$.}
 \[
 \begin{split}
  \sum_{t=1}^T \pa{\sum_{i=1}^K p_{t,i} \tloss_{t,i} - \tloss_{t,J_t^*}} \le \frac{2\bar{S}\log K}{\eta} - 
\frac{1}{\eta}\log\pa{\alpha^S(1\!-\!\alpha)^{T-\bar{S}}} + 
  \frac{\eta}{2} \sum_{t=1}^T  \sum_{i=1}^K p_{t,i} \pa{\tloss_{t,i}}^2.
 \end{split}
 \]
 To proceed, let us apply Lemma~\ref{lem:fixbound} to obtain that
\[
 \sum_{t=1}^T \pa{\tloss_{t,J_t} - \loss_{t,J_t}} \le \frac{\log\bpa{\left|C(S)\right|/\delta}}{2\gamma}
\]
simultaneously holds for all sequences $J_{1:T}\in C(S)$. By standard arguments (see, e.g., the proof of Theorem 22 
in \citet{AB09}), one can show that $\left|C(S)\right|\le K^{\bar{S}}\pa{\frac{eT}{S}}^S$.
 Now, combining the above with Equation~\eqref{eq:plossbound} and $\sum_{i=1}^K p_{t,i} \tloss_{t,i}^2 \le 
\sum_{i=1}^K \tloss_{t,i}$, we get that
 \[
 \begin{split}
  \sum_{t=1}^T \pa{\loss_{t,I_t} - \loss_{t,J^*_t}} \le& \frac{2\bar{S}\log K}{\eta} - 
\frac{1}{\eta}\log\pa{\alpha^S(1-\alpha)^{T-\bar{S}}} + 
\frac{\log\bpa{T/(S\delta')}+1}{2\gamma} 
+ \pa{\frac{\eta}{2}
+ \gamma} \sum_{t=1}^T \sum_{i=1}^K \tloss_{t,i}
\\
\le& \frac{2\bar{S}\log K}{\eta} - 
\frac{1}{\eta}\log\pa{\alpha^S(1-\alpha)^{T-\bar{S}}} + 
\frac{\log\bpa{T/(S\delta')}+1}{2\gamma} 
\\
&+ \pa{\frac{\eta}{2}
+ \gamma} \sum_{i=1}^K L_{t,i} + \pa{\frac{\eta}{2} + \gamma} \frac{\log\pa{1/\delta'}}{2\gamma},
 \end{split}
 \]
 holds with probability at least $1-2\delta'$.
 where the last line follows from Lemma~\ref{lem:fixbound} and the union bound. 
 Then, after observing that the losses are bounded in $[0,1]$ and choosing $\delta' = \delta/2$, we get that
 \[
 \begin{split}
  R_T^S \le& \frac{\pa{S+1}\log K}{\eta} - \frac{1}{\eta}\log\pa{\alpha^S(1-\alpha)^{T-S-1}} + 
\frac{\pa{S+1}\log K + S\log\pa{\frac {2eT}{S\delta}}}{2\gamma} 
\\
&+ \pa{\frac{\eta}{2}
+ \gamma} K T + \pa{\frac{\eta}{2} + \gamma} \frac{\log\pa{2/\delta}}{2\gamma}
 \end{split}
 \]
 holds with probability at least $1-\delta$.
The only remaining piece required for proving the theorem is showing that
\[
 - \log\pa{\alpha^S(1-\alpha)^{T-\bar{S}}} \le S \log\pa{\frac{eT}{S}},
\]
which follows from the proof of Corollary~1 in \cite{CBGLS12}, and then substituting the choice of $\eta$ and $\gamma$.
\qed

\subsection{The proof of Theorem~\ref{thm:sideobs}}
Before we dive into the proof, we note that Lemma~\ref{lem:fixbound} does \emph{not} hold for the loss estimates 
used by this variant of \exphix due to a subtle technical issue. Precisely, in this case $\prod_{i=1}^K 
\pa{1+\hloss_{t,i}}\neq\sum_{i=1}^K \pa{1+\hloss_{t,i}}$ prevents us from directly applying Lemma~\ref{lem:fixbound}. 
However, Corollary~\ref{cor:allbound} can still be proven exactly the same way as done in Section~\ref{sec:apps}. The 
only effect of this change is that the term $\log(1/\delta')$ is replaced by $K\log(K/\delta')$.

Turning to the actual proof, let us fix an arbitrary $\delta'\in(0,1)$ and introduce the notation
\[
 Q_t = \sum_{i=1}^K \frac{p_{t,i}}{o_{t,i} + \gamma}.
\]
 By the standard \exph-analysis, we have
 \[
  \sum_{t=1}^T \pa{\sum_{i=1}^K p_{t,i} \tloss_{t,i} - \tloss_{t,j}} \le \frac{\log K}{\eta} + 
  \frac{\eta}{2} \sum_{t=1}^T  \sum_{i=1}^K p_{t,i} \pa{\tloss_{t,i}}^2.
 \]
 Now observe that
 \[
 \begin{split}
  \sum_{t=1}^T \sum_{i=1}^K p_{t,i} \pa{\tloss_{t,i}}^2 &=
  \sum_{t=1}^T \sum_{i=1}^K \frac{p_{t,i}}{o_{t,i} + \gamma } \cdot\tloss_{t,i}
  \\
  &\le \sum_{t=1}^T \sum_{i=1}^K \frac{p_{t,i}}{o_{t,i} + \gamma } \cdot \loss_{t,i} + \frac{K\log(K/\delta')}{2\gamma}
  \\
  &\le \sum_{t=1}^T Q_t + \frac{K\log(K/\delta')}{2\gamma},
 \end{split}
 \]
 holds with probability at least $1-\delta'$ by an application of Corollary~\ref{cor:allbound} for all $i$ and taking a 
union bound.
Furthermore, we have
\[
 \begin{split}
  \sum_{i=1}^K p_{t,i} \tloss_{t,i} &=
  \sum_{i=1}^K p_{t,i} \loss_{t,i} + \sum_{i=1}^K  \pa{O_{t,i}-o_{t,i} -\gamma}\frac{p_{t,i}\loss_{t,i}}{o_{t,i} + 
\gamma}
  \\
  &\ge \sum_{i=1}^K p_{t,i} \loss_{t,i} + \sum_{i=1}^K  \pa{O_{t,i}-o_{t,i}}\frac{p_{t,i}\loss_{t,i}}{o_{t,i} + \gamma} 
- \gamma Q_t.
 \end{split}
\]
By the Hoeffding--Azuma inequality, we have
\[
 \sum_{t=1}^T \loss_{t,I_t} \le \sum_{t=1}^T \sum_{i=1}^K p_{t,i} \loss_{t,i} + \sqrt{\frac{T\log(1/\delta')}{2}}
\]
with probability at least $1-\delta'$.
After putting the above inequalities together and applying Lemma~\ref{lem:fixbound}, we obtain the bound
\[
\begin{split}
 R_T \le& \frac{\log K}{\eta} + \frac{\log(K/\delta')}{2\gamma} + \pa{\frac \eta 2 + \gamma}\sum_{t=1}^T Q_t 
 + \frac{\eta}{2} \cdot \frac{K\log(K/\delta')}{2\gamma} + \sqrt{\frac{T\log(1/\delta')}{2}}
 \\
 & + \sum_{t=1}^T \sum_{i=1}^K  \pa{o_{t,i}-O_{t,i}}\frac{p_{t,i}\loss_{t,i}}{o_{t,i}+\gamma}
\end{split}
\]
that holds with probability at least $1-3\delta'$ by the union bound.
To bound the last term on the right hand side, observe that
\[
 X_{t} = \sum_{i=1}^K \pa{o_{t,i}-O_{t,i}}\frac{p_{t,i}\loss_{t,i}}{o_{t,i}+\gamma}
\]
is a martingale-difference sequence for all $i\in[K]$ with $|X_{t}|\le K$ and conditional variance
\[
\begin{split}
\sigma_t^2\pa{X_{t}} =& \EEcc{\pa{\sum_{i=1}^K \pa{o_{t,i}-O_{t,i}}\frac{p_{t,i}}{o_{t,i} + \gamma}}^2}{\F_{t-1}}
\\
\le& \EEcc{\pa{\sum_{i=1}^K O_{t,i} \frac{p_{t,i}}{o_{t,i} + \gamma}}^2}{\F_{t-1}} \quad\qquad\qquad\qquad\mbox{(since
$\EEcc{O_{t,i}}{\F_{t-1}} = o_{t,i}$)}
\\
=& \EEcc{\sum_{i=1}^K \sum_{j=1}^K O_{t,i}O_{t,j} \frac{p_{t,i}}{o_{t,i} + \gamma} \cdot\frac{p_{t,j}}{o_{t,j} +
\gamma}}{\F_{t-1}}
\\
\le& \EEcc{\sum_{i=1}^K \sum_{j=1}^K O_{t,i} \frac{p_{t,i}}{o_{t,i} + \gamma} \cdot\frac{p_{t,j}}{o_{t,j} +
\gamma}}{\F_{t-1}}
\quad\qquad\mbox{(since
$O_{t,j\le 1}$)}
\\
=& \sum_{i=1}^K \sum_{j=1}^K \frac{p_{t,i}o_{t,i}}{o_{t,i} + \gamma} \cdot\frac{p_{t,j}}{o_{t,j} +
\gamma}
\le \sum_{i=1}^K p_{t,i} \sum_{j=1}^K \frac{p_{t,j}}{o_{t,j} + \gamma} = Q_t.
\end{split}
\]
Thus, an application of Freedman's inequality (see, e.g., Theorem~1 of \citet{BLLRS11}), we can thus obtain the bound
\[
 \sum_{t=1}^T X_{t} \le \frac{\log(1/\delta')}{\omega} + (e-2) \omega \sum_{t=1}^T Q_t
\]
that holds with probability at least $1-\delta'$ for all $\omega \le 1/K$. 
Combining this result with the previous bounds and using the union bound, we arrive at the bound
\[
\begin{split}
 R_T \le& \frac{\log K}{\eta} + \frac{\log(K/\delta')}{2\gamma} + \frac{\log(1/\delta')}{\omega} + \pa{\frac \eta 2 +
\gamma + \omega}\sum_{t=1}^T Q_t 
 + \frac{\eta}{2} \cdot \frac{K\log(K/\delta')}{2\gamma} + \sqrt{\frac{T\log(1/\delta')}{2}}
\end{split}
\]
that holds with probability at least $1-4\delta'$. 

Invoking Lemma~1 of \citet{KNVM14} that states that
$$\sum_{i=1}^K \frac{p_{t,i}}{o_{t,i} + \gamma}\leq2\alpha\log\left(1+\frac{\lceil
K^2/\gamma\rceil+K}{\alpha}\right)+2$$
holds almost surely and setting $\delta' = \delta/4$, we obtain the bound
 \[
\begin{split}
 R_T \le& \frac{\log K}{\eta} + \frac{\log(4K/\delta)}{2\gamma} + \frac{\log(4/\delta)}{\omega} + \pa{\eta +
2\gamma + 2\omega}\alpha'T
 + \frac{\eta}{2} \cdot \frac{K\log(4K/\delta)}{2\gamma} + \sqrt{\frac{T\log(4/\delta)}{2}}
\end{split}
\]
that holds with probability at least $1-\delta$, where $\alpha' = \alpha\log\left(1+\frac{\lceil 
K^2/\gamma\rceil+K}{\alpha}\right)+1$.

Now notice that when setting $\eta = 2\gamma = \sqrt{\frac{\log K}{2\alpha T \log(KT)}}$ and $\omega =
\sqrt{\frac{\log(4/\delta)}{2\alpha T \log(KT)}}$, we have $\alpha' \le 2\alpha 
\log(KT)$ and the above bound becomes
\[
 \begin{split}
  R_T \le& \pa{4 + 2\sqrt{\log\pa{4/\delta}}}\cdot\sqrt{2\alpha T \pa{\log^2 K + \log KT}} + \sqrt{\frac{2\alpha
T\log(KT)}{\log K}} \log\pa{4/\delta} + 
\\
&+ \sqrt{\frac{T\log(4/\delta)}{2}} + \frac{K\log\pa{4K/\delta}}{2}.
 \end{split}
 \]
 The proof is concluded by observing that the last term is bounded by the third one if $T\ge K^2/(8\alpha)$.
\qed

\end{document}